\documentclass[conference]{IEEEtran}
\usepackage{times}

\usepackage[numbers, square, comma, sort&compress]{natbib}
\usepackage{multicol}
\usepackage[bookmarks=true]{hyperref}
\usepackage[utf8]{inputenc} 
\usepackage[T1]{fontenc}    
\usepackage{hyperref}       
\usepackage{url}            
\usepackage{booktabs}       
\usepackage{amsfonts}       
\usepackage{nicefrac}       
\usepackage{microtype}      
\usepackage{epigraph}
\usepackage{graphicx}
\usepackage{amsmath}
\usepackage{booktabs}
\usepackage{array}
\usepackage{hyperref}
\usepackage[binary-units=true]{siunitx}
\usepackage{subcaption}
\usepackage{bm}
\usepackage{mathtools}
\usepackage{cuted}
\usepackage{etoolbox}
\usepackage{balance}

\usepackage{wrapfig}

\usepackage{pgfplots}
\pgfplotsset{compat=newest}
\usetikzlibrary{plotmarks}
\usetikzlibrary{babel}
\usetikzlibrary{arrows.meta}
\usepgfplotslibrary{patchplots}
\usepackage{grffile}
\usepackage{amsmath}

\usepgfplotslibrary{groupplots}
\usepgfplotslibrary{dateplot}

\usepackage{capt-of}

\usepackage{xcolor}
\definecolor{rebutall}{rgb}{0.0,0.0,0.0}
\newcommand\rebutall[1]{\textcolor{rebutall}{ #1}}

\newtheorem{theorem}{Theorem}
\newtheorem{lemma}[theorem]{Lemma}

\IEEEoverridecommandlockouts

\definecolor{somegray}{rgb}{0.5, 0.5, 0.5}
\newcommand{\darkgrayed}[1]{\textcolor{somegray}{#1}}
\makeatletter
\newcommand*\titleheader[1]{\gdef\@titleheader{#1}}
\AtBeginDocument{%
  \let\st@red@title\@title
  \def\@title{%
    \vskip-1.5em
    \bgroup\normalfont\large\centering\@titleheader\par\egroup
    \vskip0.5em\st@red@title}
}
\makeatother

\titleheader{\darkgrayed{This paper has been accepted for publication at Robotics: Science and Systems, 2020.}}

\title{Deep Drone Acrobatics} 

\begin{document}

\author{Elia Kaufmann\IEEEauthorrefmark{1}\IEEEauthorrefmark{3},
        Antonio Loquercio\IEEEauthorrefmark{1}\IEEEauthorrefmark{3}\thanks{\IEEEauthorrefmark{1}Equal contribution.},
		Ren\'{e} Ranftl\IEEEauthorrefmark{2},
		Matthias M\"uller\IEEEauthorrefmark{2}\thanks{\IEEEauthorrefmark{3}Robotics and Perception Group, University of Zurich and ETH Zurich.},
		Vladlen Koltun\IEEEauthorrefmark{2},
        Davide Scaramuzza\IEEEauthorrefmark{3}\thanks{\IEEEauthorrefmark{2}Intelligent Systems Lab, Intel.}
}

\makeatletter
\g@addto@macro\@maketitle{
  \def\mycolspace{1.2mm}
 	\begin{tabular}{@{}c@{\hspace{\mycolspace}}c@{\hspace{\mycolspace}}c@{}}
\includegraphics[width=0.329\textwidth,trim=0 3.63cm 0 0.5cm, clip]{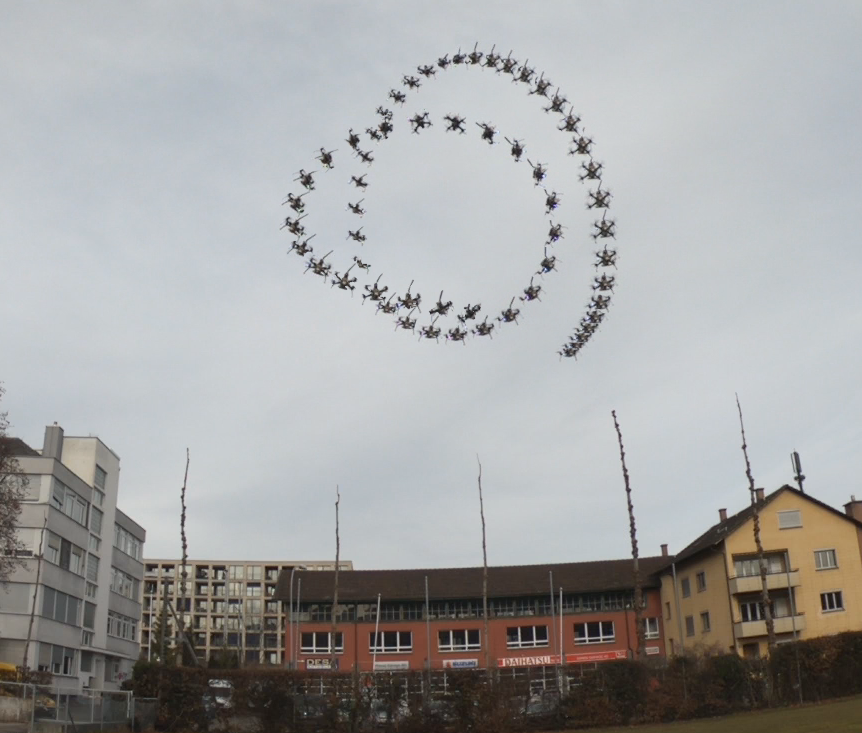} &
\includegraphics[width=0.329\textwidth]{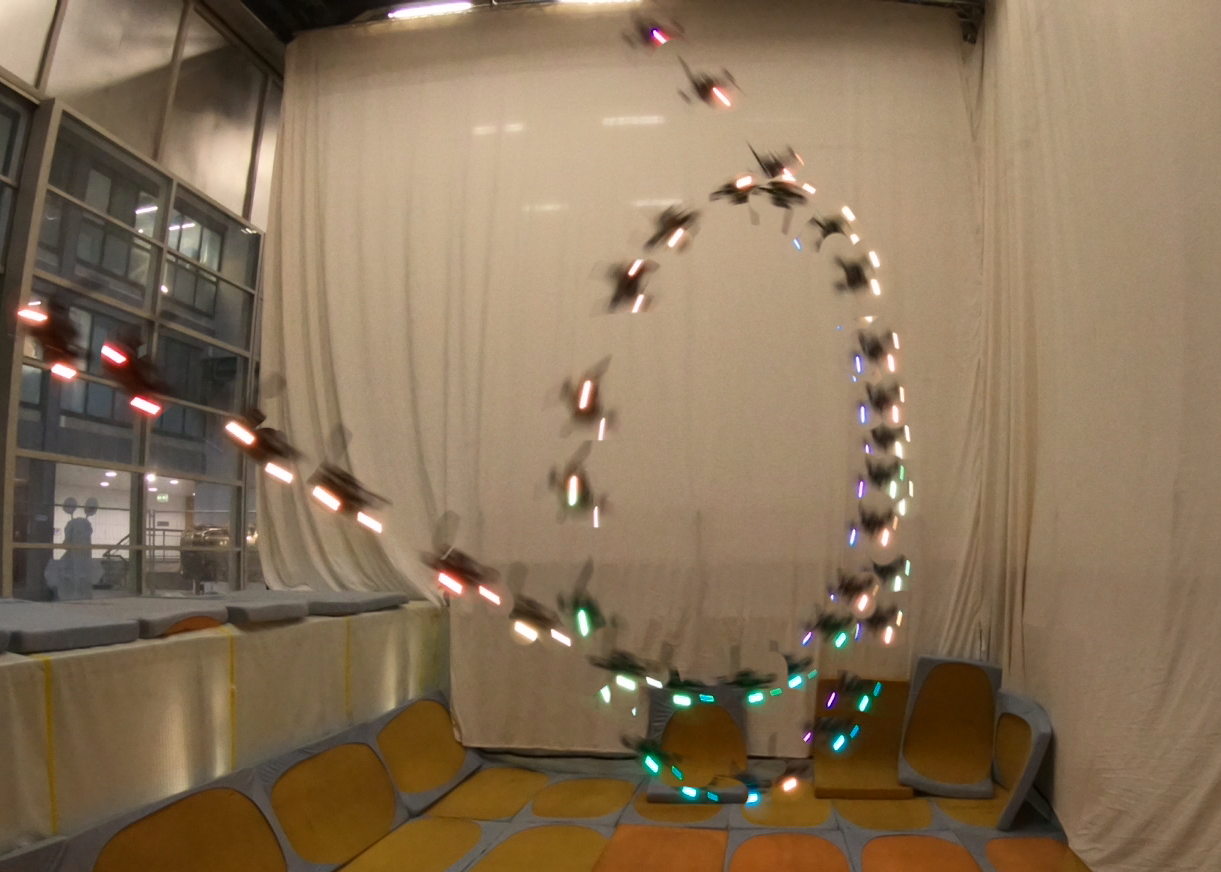} &
\includegraphics[width=0.329\textwidth,trim=0 5.9cm 5.9cm 0, clip]{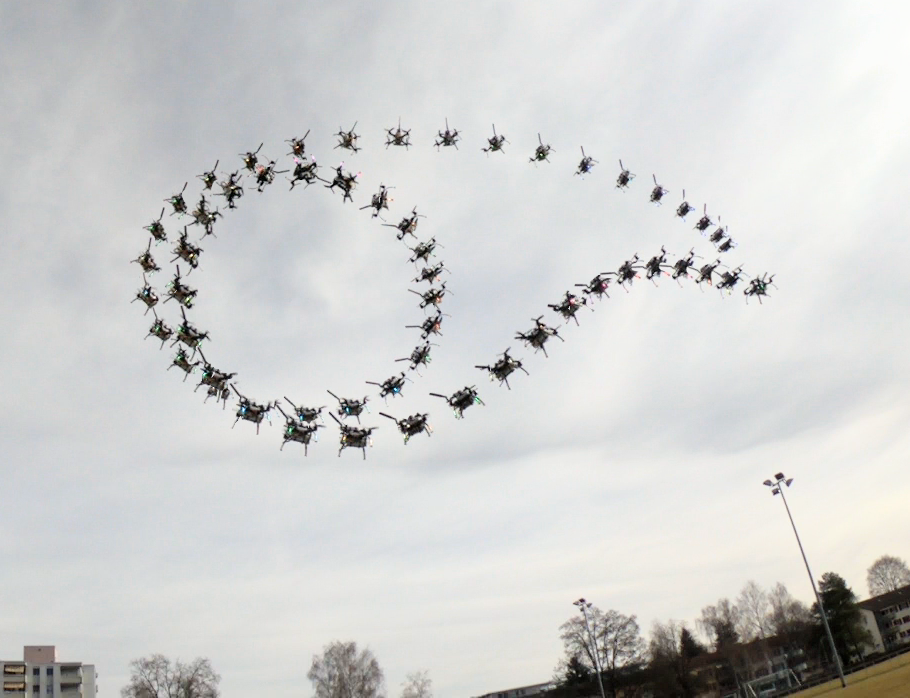}
	\end{tabular}
	\captionof{figure}{A quadrotor performs a Barrel Roll (left), a Power Loop (middle), and a Matty Flip (right). We safely train acrobatic controllers in simulation and deploy them with no fine-tuning (\emph{zero-shot transfer}) on physical quadrotors. The approach uses only onboard sensing and computation. No external motion tracking was used.
	\label{fig:catcheye}}
}
\makeatother
\maketitle

\begin{abstract}
Performing acrobatic maneuvers with quadrotors is extremely challenging.
Acrobatic flight requires high thrust and extreme angular accelerations that push the platform to its physical limits.
Professional drone pilots often measure their level of mastery by flying such maneuvers in competitions.
In this paper, we propose to learn a sensorimotor policy that enables an autonomous quadrotor to fly extreme acrobatic maneuvers with only onboard sensing and computation.
We train the policy entirely in simulation by leveraging demonstrations from an optimal controller that has access to privileged information. We use appropriate abstractions of the visual input to enable transfer to a real quadrotor.
We show that the resulting policy can be directly deployed in the physical world without any fine-tuning on real data.
Our methodology has several favorable properties: it does not require a human expert to provide demonstrations, it cannot harm the physical system during training, and it can be used to learn maneuvers that are challenging even for the best human pilots. 
Our approach enables a physical quadrotor to fly maneuvers such as the Power Loop, the Barrel Roll, and the Matty Flip, during which it incurs accelerations of up to 3g. 
\end{abstract}

\section*{Supplementary Material}
\rebutall{A video demonstrating acrobatic maneuvers is available at \url{https://youtu.be/2N_wKXQ6MXA}.
Code can be found at \\ \url{https://github.com/uzh-rpg/deep_drone_acrobatics}.
}

\section{Introduction}
Acrobatic flight with quadrotors is extremely challenging. Human drone pilots require many years of practice to safely master maneuvers such as power loops and barrel rolls\footnote{\href{https://www.youtube.com/watch?v=T1vzjPa5260}{https://www.youtube.com/watch?v=T1vzjPa5260}}.
Existing autonomous systems that perform agile maneuvers require external sensing and/or external computation~\cite{lupashin2010simple, abbeel2010autonomous,  bry2015aggressive}. For aerial vehicles that rely only on onboard sensing and computation, the high accelerations that are required for acrobatic maneuvers together with the unforgiving requirements on the control stack raise fundamental questions in both perception and control. Therefore, they provide a natural  benchmark to compare the capabilities of autonomous systems against trained human pilots.

Acrobatic maneuvers represent a challenge for the actuators, the sensors, and all physical components of a quadrotor. 
While hardware limitations can be resolved using expert-level equipment that allows for extreme accelerations, the major limiting factor to enable agile flight is reliable state estimation.
Vision-based state estimation systems either provide significantly reduced accuracy or completely fail at high accelerations due to effects such as motion blur, large displacements, and the difficulty of robustly tracking features over long time frames~\cite{Cadena2016}.
Additionally, the harsh requirements of fast and precise control at high speeds make it difficult to tune controllers on the real platform, since even tiny mistakes can result in catastrophic outcomes for the platform.

The difficulty of agile autonomous flight led previous work to mostly focus on specific aspects of the problem.
One line of research focused on the control problem, assuming near-perfect state estimation from external sensors~\cite{lupashin2010simple, abbeel2010autonomous, hwangbo2017control, bry2015aggressive}.
While these works showed impressive examples of agile flight, they focused purely on control. The issues of reliable perception and state estimation during agile maneuvers were cleverly circumvented by instrumenting the environment with sensors (such as Vicon and OptiTrack) that provide near-perfect state estimation to the platform at all times.
Recent works addressed the control and perception aspects in an integrated way via techniques like perception-guided trajectory optimization~\cite{falanga2018pampc,aggressive_falanga,Shen2013aggressive} or training end-to-end visuomotor agents~\cite{ZhangKLA16}. However, acrobatic performance of high-acceleration maneuvers with only onboard sensing and computation has not yet been achieved.

\rebutall{In this paper, we show for the first time that a vision-based autonomous quadrotor with only onboard sensing and computation is capable of autonomously performing agile maneuvers with accelerations of up to 3g, as shown in Fig.~\ref{fig:catcheye}.
This contribution is enabled by a novel simulation-to-reality transfer strategy, which is based on abstraction of both visual and inertial measurements.  
We demonstrate both formally and empirically that the presented abstraction strategy decreases the simulation-to-reality gap with respect to a naive use of sensory inputs.
Equipped with this strategy, we train an end-to-end sensimotor controller to fly acrobatic maneuvers \emph{exclusively} in simulation.
Learning agile maneuvers entirely in simulation has several advantages: (i)~Maneuvers can be simply specified by reference trajectories in simulation and do not require expensive demonstrations by a human pilot, 
(ii)~training is safe and does not pose any physical risk to the quadrotor, and
(iii)~the approach can scale to a large number of diverse maneuvers, including ones that can only be performed by the very best human pilots.
}

\rebutall{
Our sensorimotor policy is represented by a neural network that combines information from different input modalities to directly regress thrust and body rates. To cope with different output frequencies of the onboard sensors, we design an asynchronous network that operates independently of the sensor frequencies. This network is trained in simulation to imitate demonstrations from an optimal controller that has access to privileged state information.
}

We apply the presented approach to learning autonomous execution of three acrobatic maneuvers that are challenging even for expert human pilots: the Power Loop, the Barrel Roll, and the Matty Flip.
Through controlled experiments in simulation and on a real quadrotor, we show that the presented approach leads to robust and accurate policies that are able to reliably perform the maneuvers with only onboard sensing and computation.

\section{Related Work}
Acrobatic maneuvers comprehensively challenge perception and control stacks. The agility that is required to perform acrobatic maneuvers requires carefully tuned controllers together with accurate state estimation. Compounding the challenge, the large angular rates and high speeds that arise during the execution of a maneuver induce strong motion blur in vision sensors and thus compromise the quality of state estimation.

The complexity of the problem has led early works to only focus on the control aspect while disregarding the question of reliable perception.
Lupashin et al.~\cite{lupashin2010simple} proposed iterative learning of control strategies to enable platforms to perform multiple flips.
Mellinger et al.~\cite{mellinger2012trajectory} used a similar strategy to autonomously fly quadrotors through a tilted window~\cite{mellinger2012trajectory}.
By switching between two controller settings,  Chen et al.~\cite{chen2019controller} also demonstrated multi-flip maneuvers.
Abbeel~et~al.~\cite{abbeel2010autonomous} learned to perform a series of acrobatic maneuvers with autonomous helicopters. Their algorithm leverages expert pilot demonstrations to learn task-specific controllers.
While these works proved the ability of flying machines to perform agile maneuvers, they did not consider the perception problem. Indeed, they all assume that near-perfect state estimation is available during the maneuver, which in practice requires instrumenting the environment with dedicated sensors.

Aggressive flight with only onboard sensing and computation is an open problem.
The first attempts in this direction were made by Shen et al.~\cite{Shen2013aggressive}, who demonstrated agile vision-based flight. 
The work was limited to low-acceleration trajectories, therefore only accounting for part of the control and perception problems encountered at high speed.
More recently, Loianno et al.~\cite{loianno2016estimation} and Falanga et al.~\cite{aggressive_falanga} demonstrated aggressive flight through narrow gaps with only onboard sensing.
Even though these maneuvers are agile, they are very short and cannot be repeated without re-initializing the estimation pipeline.
Using perception-guided optimization, Falanga et al.~\cite{falanga2018pampc} and Lee et al.~\cite{lee2020aggressive} proposed a model-predictive control framework to plan aggressive trajectories while minimizing motion blur. However, such control strategies are too conservative to fly acrobatic maneuvers, which always induce motion blur. 

Abolishing the classic division between perception and control, a recent line of work proposes to train visuomotor policies directly from data.
Similarly to our approach, Zhang et al.~\cite{ZhangKLA16} trained a neural network from demonstrations provided by an MPC controller.
While the latter has access to the full state of the platform and knowledge of obstacle positions, the network only observes laser range finder readings and inertial measurements.
Similarly, Li et al.~\cite{li2019aggressive} proposed an imitation learning approach for training visuomotor agents for the task of quadrotor flight.
The main limitation of these methods is in their sample complexity: large amounts of demonstrations are required to fly even straight-line trajectories.
As a consequence, these methods were only validated in simulation or were constrained to slow hover-to-hover trajectories.

Our approach employs \emph{abstraction} of sensory input~\cite{muller2018driving} to reduce the problem's sample complexity and enable \emph{zero-shot sim-to-real transfer}.
\rebutall{While prior work has demonstrated the possibility of controlling real-world quadrotors with zero-shot sim-to-real transfer~\cite{loquercio2019deep, sadeghi2016cad2rl}, our approach is the first to learn an end-to-end sensorimotor mapping~-- from sensor measurements to low-level controls~-- that can perform high-speed and high-acceleration acrobatic maneuvers on a real physical system.}
\section{Overview}

In order to perform acrobatic maneuvers with a quadrotor, we train a sensorimotor controller to predict low-level actions from a history of onboard sensor measurements and a user-defined reference trajectory. 
An observation $\bm{o}[k] \in \mathbb{O}$ at time $k \in [0,\dots, T]$ consists of a camera image $\mathcal{I}[k]$  and an inertial measurement $\phi[k]$.
Since the camera and IMU typically operate at different frequencies, the visual and inertial observations are updated at different rates.
The controller's output is an action ${\bm{u}[k]= [c, \bm{\omega}^\top]^\top \in \mathbb{U}}$ that consists of continuous mass-normalized collective thrust $c$ and bodyrates ${\bm{\omega}}=[\omega_x, \omega_y, \omega_z]^\top$ that are defined in the quadrotor body frame.

The controller is trained via \emph{privileged learning}~\cite{chen2019learning}. Specifically, the policy is trained on demonstrations that are provided by a privileged expert: an optimal controller that has access to privileged information that is not available to the sensorimotor student, such as the full ground-truth state of the platform ${\bm{s}[k] \in \mathbb{S}}$.
The privileged expert is based on a classic optimization-based planning and control pipeline that tracks a reference trajectory from the state $\bm{s}[k]$ using MPC~\cite{falanga2018pampc}.

We collect training demonstrations from the privileged expert in simulation.
Training in simulation enables synthesis and use of unlimited training data for any desired trajectory, without putting the physical platform in danger. This includes maneuvers that stretch the abilities of even expert human pilots.
To facilitate zero-shot simulation-to-reality transfer, the sensorimotor student does not directly access raw sensor input such as color images. Rather, the sensorimotor controller acts on an \emph{abstraction} of the input, in the form of feature points extracted via classic computer vision. Such abstraction supports sample-efficient training, generalization, and simulation-to-reality transfer~\cite{muller2018driving,Zhou2019DoesCV}.

The trained sensorimotor student does not rely on any privileged information and can be deployed directly on the physical platform. We deploy the trained controller to perform acrobatic maneuvers in the physical world, with no adaptation required.

The next section presents each aspect of our method in detail.
\section{Method}

We define the task of flying acrobatic maneuvers with a quadrotor as a discrete-time, continuous-valued optimization problem.
Our task is to find an end-to-end control policy $\pi \colon  \mathbb{O} \to \mathbb{U}$, defined by a neural network, which minimizes the following finite-horizon objective:
\begin{equation}
    \min_{\pi}\\ J(\pi) = \mathbb{E}_{\rho(\pi)}  \Big[\sum_{k=0}^{k=T} \mathcal{C}(\bm{\tau}_r[k], \bm{s}[k]) \Big],
    \label{eq:prob}
\end{equation}
where $\mathcal{C}$ is a quadratic cost depending on a reference trajectory $\bm{\tau}_r[k]$ and the quadrotor state $\bm{s}[k]$, and $\rho(\pi)$ is the distribution of possible trajectories $\lbrace (\bm{s}[0], \bm{o}[0], \bm{u}[0]), \dots, (\bm{s}[T], \bm{o}[T], \bm{u}[T])\rbrace$ induced by the policy $\pi$.

We define the quadrotor state $\bm{s}[k] = [\bm{p},\bm{q},\bm{v},\bm{\omega}]$ as the platform position $\bm{p}$, its orientation quaternion $\bm{q}$, and their derivatives. Note that the agent $\pi$ does not directly observe the state $\bm{s}[k]$.
We further define the reference trajectory $\bm{\tau}_r[k]$ 
as a time sequence of reference states which describe the desired trajectory.
We formulate the cost $\mathcal{C}$ as
\begin{equation}
    \mathcal{C}(\bm{\tau}_r[k], \bm{s}[k]) =  \bm{x}[k]^\top\mathcal{L}\bm{x}[k],
    \label{eq:one_step_cost}
\end{equation}
where $\bm{x}[k] = \bm{\tau}_r[k] - \bm{s}[k]$ denotes the difference between the state of the platform and the corresponding reference at time $k$, and $\mathcal{L}$ is a positive-semidefinite cost matrix.

\begin{figure*}[t]
\centering
  \includegraphics[width=0.24\linewidth]{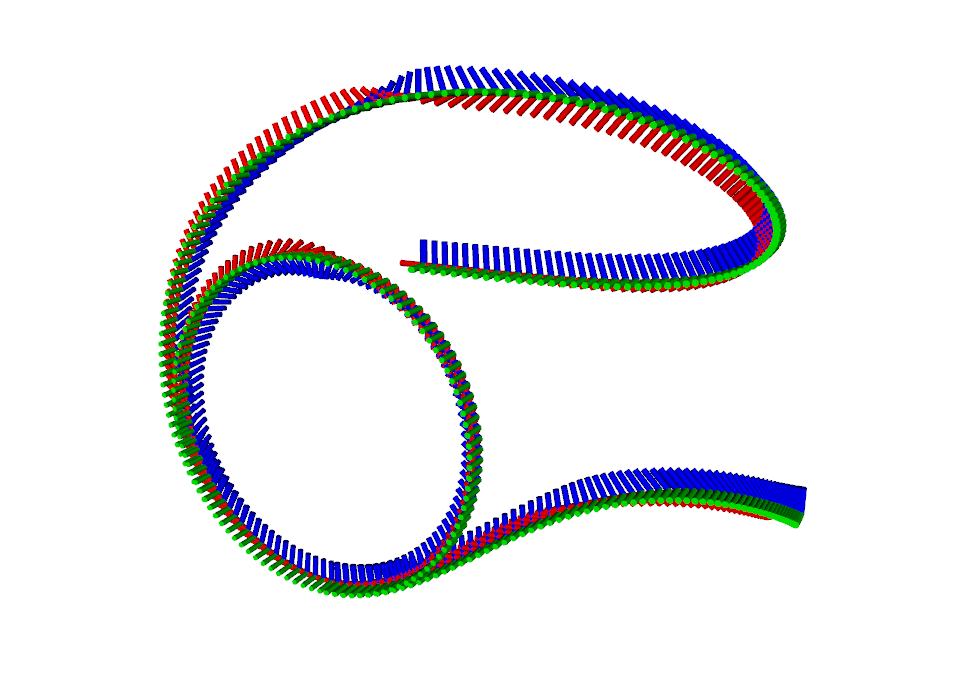}
  \includegraphics[width=0.24\textwidth]{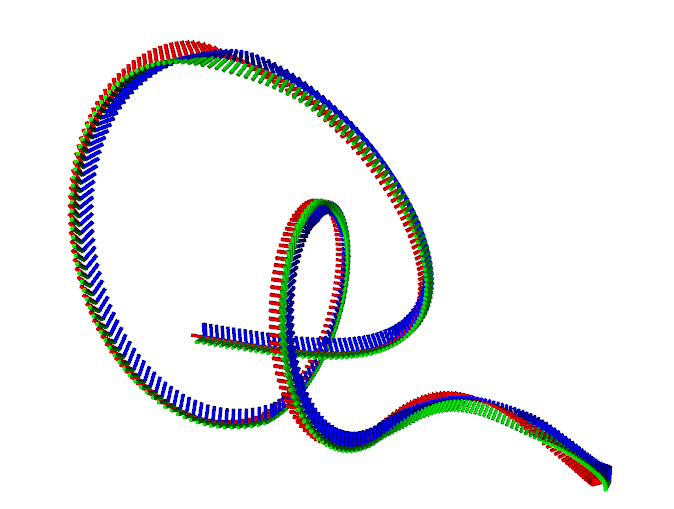}
  \includegraphics[width=0.24\textwidth]{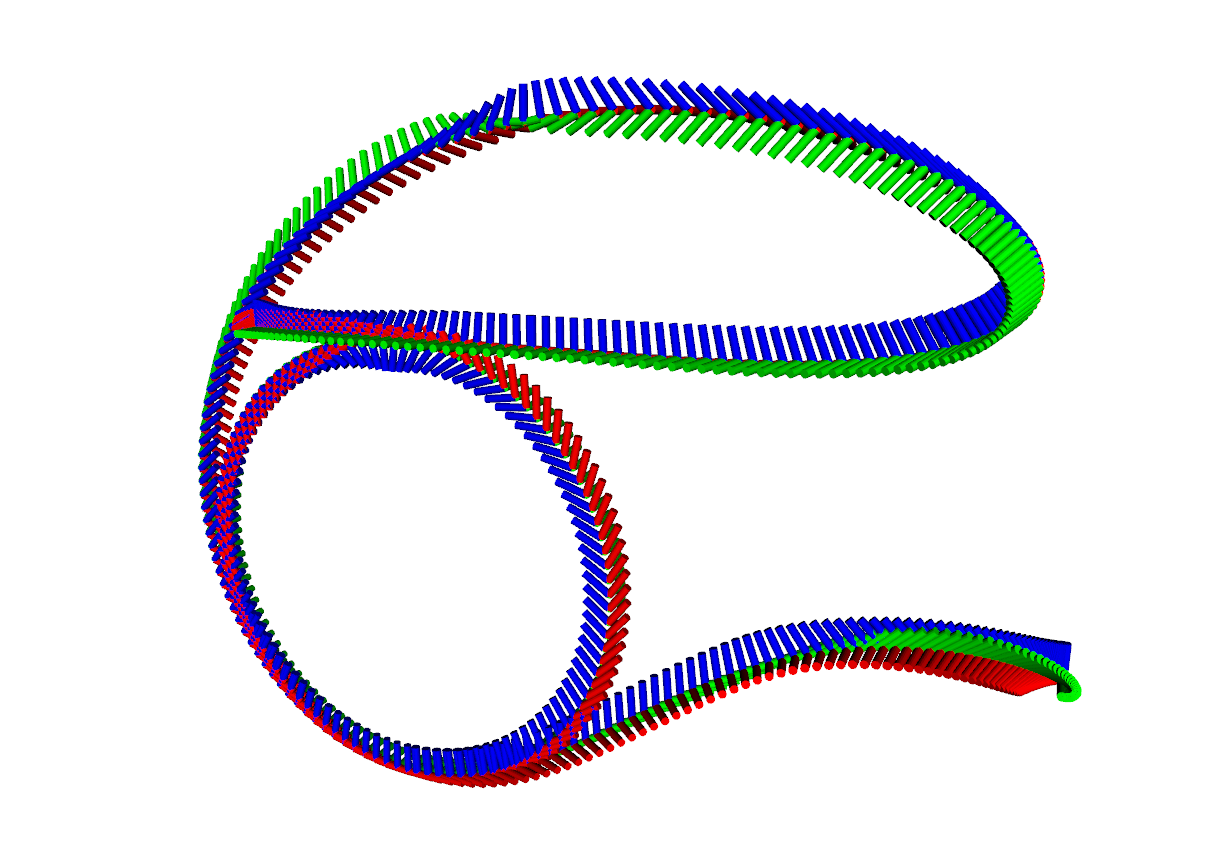}
  \includegraphics[width=0.68\textwidth]{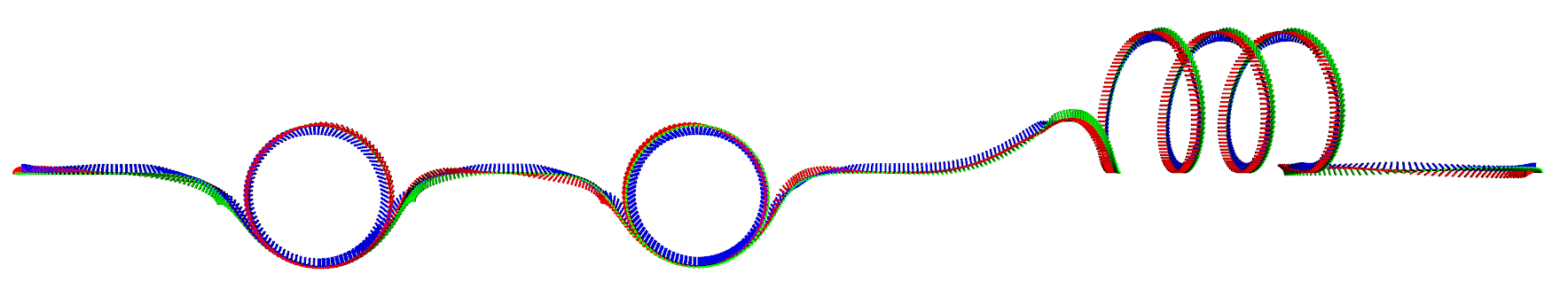}
\caption{Reference trajectories for acrobatic maneuvers. Top row, from left to right: Power Loop, Barrel Roll, and Matty Flip. Bottom row: Combo.}
\label{fig:maneuvers}
\end{figure*}

\subsection{Reference Trajectories}
Both the privileged expert and the learned policy assume access to a reference trajectory $\bm{\tau}_r[k]$ that specifies an acrobatic maneuver.
To ensure that such reference is dynamically feasible, it
has to satisfy constraints that are imposed by the physical limits and the underactuated nature of the quadrotor platform. 
Neglecting aerodynamic drag and motor dynamics, the dynamics of the quadrotor can be modelled as
\begin{equation}\label{eq:sys_dyn}
    \begin{aligned}
    \dot{\bm{p}}_{\text{WB}} &= \bm{v}_{\text{WB}} \\
    \dot{\bm{v}}_{\text{WB}} &= \prescript{}{{\text{W}}}{\bm{g}} + \bm{q}_{\text{WB}} \odot \bm{c}_{\text{B}} \\
    \dot{\bm{q}}_{\text{WB}} &= \frac{1}{2}\Lambda\left( \bm{\omega}_{\text{B}}\right)\cdot\bm{q}_{\text{WB}}\\
    \dot{\bm{\omega}_\text{B}} &= \bm{J}^{-1}\cdot \left(\bm{\eta} - \bm{\omega}_{\text{B}}\times\bm{J}\cdot \bm{\omega}_{\text{B}}\right) \; ,
    \end{aligned}
\end{equation}
where $\bm{p}_{\text{WB}}$, $\bm{v}_{\text{WB}}$, $\bm{q}_{\text{WB}}$  denote the position, linear velocity, and orientation of the platform body frame with respect to the world frame.
The gravity vector $\prescript{}{{\text{W}}}{\bm{g}}$ is expressed in the world frame and $\bm{q}_{\text{WB}}\odot\bm{c}_\text{B}$ denotes the rotation of the mass-normalized thrust vector $\bm{c}_{\text{B}} = (0, 0, c)^\top$ by quaternion $\bm{q}_{\text{WB}}$.
The time derivative of a quaternion $\bm{q} = (q_w, q_x, q_y, q_z)^\top$ is given by ${\dot{\bm{q}} = \frac{1}{2}\Lambda(\bm{\omega})\cdot \bm{q}}$ and $\Lambda(\bm{\omega})$ is a skew-symmetric matrix of the vector $(0, \bm{\omega}^\top)^\top = (0, \omega_x, \omega_y, \omega_z)^\top$.
The diagonal matrix ${\bm{J} = \text{diag}(J_{xx}, J_{yy}, J_{zz})}$ denotes the quadrotor inertia, and $\bm{\eta}\in\mathbb{R}^{3}$ are the torques acting on the body due to the motor thrusts. 

Instead of directly planning in the full state space, 
we plan reference trajectories in the space of \emph{flat outputs} ${\bm{z} = [x,y,z, \psi]^\top}$ proposed in~\cite{mellinger2011minimum}, where $x,y,z$ denote the position of the quadrotor and $\psi$ is the yaw angle.
It can be shown that any smooth trajectory in the space of flat outputs can be tracked by the underactuated platform (assuming reasonably bounded derivatives). 

The core part of each acrobatic maneuver is a circular motion primitive with constant tangential velocity $\bm{v}_l$.
The orientation of the quadrotor is constrained by the thrust vector the platform needs to produce. Consequently, the desired platform orientation is undefined when there is no thrust. 
To ensure a well-defined reference trajectory through the whole circular maneuver, we constrain the norm of the tangential velocity $\bm{v}_l$ to be larger by a margin $\varepsilon$ than the critical tangential velocity that would lead to free fall at the top of the maneuver:
\begin{equation}
    \Vert \bm{v}_l \Vert > \varepsilon \sqrt{rg},
\end{equation}
where $r$ denotes the radius of the loop, $g = \SI{9.81}{\meter\per\second\squared}$, and $\varepsilon = 1.1$.

While the circular motion primitives form the core part of the agile maneuvers, we use constrained polynomial trajectories to enter, transition between, and exit the maneuvers.
A polynomial trajectory is described by four polynomial functions specifying the independent evolution of the components of $\bm{z}$ over time:
\begin{equation}
    z_i(t) = \sum_{j=0}^{j=P_i}a_{ij}\cdot t^j \quad\text{for}\quad i\in \lbrace 0,1,2,3 \rbrace\;.
\end{equation}
We use polynomials of order $P_i = 7$ for the position components $(i=\lbrace 0, 1, 2\rbrace)$ of the flat outputs and $P_i=2$ for the yaw component $(i=3)$.
By enforcing continuity for both start ($t=0$) and end ($t=t_m$) of the trajectory down to the 3rd derivative of position, the trajectory is fully constrained.
We minimize the execution time $t_m$ of the polynomial trajectories, while constraining the maximum speed, thrust, and body rates throughout the maneuver. 

Finally, the trajectories are concatenated to the full reference trajectory, which is then converted back to the full state-space representation $\bm{\tau}_r(t)$~\cite{mellinger2011minimum}. 
Subsequent sampling with a frequency of $\SI{50}{\hertz}$ results in the discrete-time representation $\bm{\tau}_r[k]$ of the reference trajectory. Some example trajectories are illustrated in Figure~\ref{fig:maneuvers}.

\subsection{Privileged Expert}
\label{sec:mpc}

Our privileged expert $\pi^*$ consists of an MPC~\cite{falanga2018pampc} that generates collective thrust and body rates via an optimization-based scheme.
The controller operates on the simplified dynamical model of a quadrotor proposed in~\cite{mueller2013computationally}: 
\begin{equation}\label{eq:sys_dyn_mpc}
    \begin{aligned}
    \dot{\bm{p}}_{\text{WB}} &= \bm{v}_{\text{WB}} \\
    \dot{\bm{v}}_{\text{WB}} &= \prescript{}{{\text{W}}}{\bm{g}} + \bm{q}_{\text{WB}} \odot \bm{c}_{\text{B}} \\
    \dot{\bm{q}}_{\text{WB}} &= \frac{1}{2}\Lambda\left( \bm{\omega}_B\right)\cdot\bm{q}_{\text{WB}}
    \end{aligned}
\end{equation}
In contrast to the model \eqref{eq:sys_dyn}, the simplified model neglects the dynamics of the angular rates. 
The MPC repeatedly optimizes the open-loop control problem over a receding horizon of $N$ time steps and applies the first control command from the optimized sequence. 
Specifically, the action computed by the MPC is the first element of the solution to the following optimization problem: 
\begin{equation}\label{eq:mpc_criterion}
\begin{split}
    \pi^\ast = \min_{\bm{u}}&\biggl\lbrace \bm{x}[N]^\top\mathcal{Q}\bm{x}[N] \\
    &+ \sum_{k=1}^{N-1}\left(\bm{x}[k]^\top\mathcal{Q}\bm{x}[k] + \bm{u}[k]^\top\mathcal{R}\bm{u}[k] \right)\biggr\rbrace \\
    \text{s.t.} \quad &\bm{r}(\bm{x}, \bm{u}) = 0 \\
    &\bm{h}(\bm{x}, \bm{u}) \leq 0    , 
\end{split}
\end{equation}
where $\bm{x}[k] = \bm{\tau}_r[k] - \bm{s}[k]$ denotes the difference between the state of the platform at time $k$ and the corresponding reference $\bm{\tau}_r[k]$, $\bm{r}(\bm{x}, \bm{u})$ are equality constraints imposed by the system dynamics \eqref{eq:sys_dyn_mpc}, and $\bm{h}(\bm{x}, \bm{u})$ are optional bounds on inputs and states. $\mathcal{Q},\mathcal{R}$ are positive-semidefinite cost matrices.

\begin{figure*}[t]
    \centering
    \includegraphics[width=1.0\textwidth]{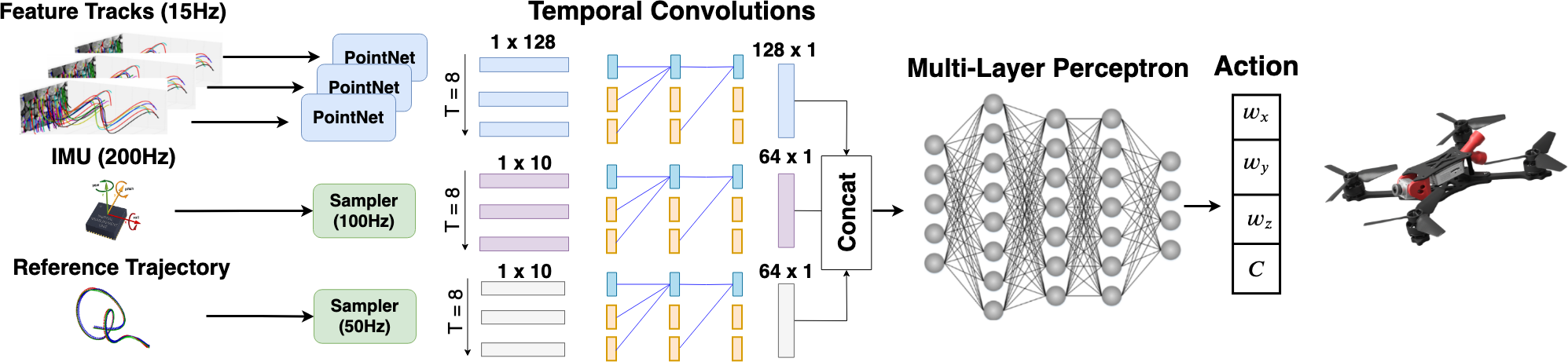}
    \caption{Network architecture. The network receives a history of feature tracks, IMU measurements, and reference trajectories as input. Each input modality is processed using temporal convolutions and updated at different input rates. The resulting intermediate representations are processed by a multi-layer perceptron at a fixed output rate to produce collective thrust and body rate commands.}
    \label{fig:network}
\end{figure*}

\subsection{Learning}

The sensorimotor controller is trained by imitating demonstrations provided by the privileged expert. While the expert has access to privileged information in the form of ground-truth state estimates, the sensorimotor controller does not access any privileged information and can be directly deployed in the physical world~\cite{chen2019learning}.

A lemma by Pan et al.~\cite{pan2018agile} formally defines an upper bound between the expert and the student performance as
\begin{align}
    J(\pi) - J(\pi^*) & \leq C_{\pi^*} \mathbb{E}_{\rho(\pi)} \big[ DW(\pi, \pi^*) \big] \nonumber\\
                     & \leq C_{\pi^*} \mathbb{E}_{\rho(\pi)} \mathbb{E}_{\small \bm{u}^* \sim \pi^*} \mathbb{E}_{\small \bm{u} \sim \pi} [\left\lVert \bm{u}^* - \bm{u} \right\rVert],
    \label{eq:supervised_learning}
\end{align}
where $DW(\cdot,\cdot)$ is the Wasserstein metric~\cite{gibbs2002choosing} and $C_{\pi^*}$ is a constant depending on the smoothness of expert actions.
Finding an agent $\pi$ with the same performance as the privileged controller $\pi^*$ boils down to minimizing the discrepancy in actions between the two policies on the expected agent trajectories $\rho(\pi)$.

The aforementioned discrepancy can be minimized by an iterative supervised learning process known as DAGGER~\cite{dagger}.
This process iteratively collects data by letting the student control the platform, annotating the collected observations with the experts' actions, and updating the student controller based on the supervised learning problem
\begin{equation}
 \pi = \min_{\hat{\pi}} \mathbb{E}_{\small \bm{s}[k] \sim \rho(\pi)}[\left\lVert \bm{u}^*(\bm{s}[k]) - \hat{\pi}(\bm{o}[k]) \right\rVert],
\end{equation}
 where $\bm{u}^*(\bm{s}[k])$ is the expert action and $\bm{o}[k]$ is the observation vector in the state $\bm{s}[k]$.
Running this process for $O(N)$ iterations yields a policy $\pi$ with performance $J{(\pi) \leq J(\pi^*) + O(N)}$~\cite{dagger}.

Naive application of this algorithm to the problem of agile flight in the physical world presents two major challenges: how can the expert access the ground-truth state $\bm{s}[k]$ and how can we protect the platform from damage when the partially trained student $\pi$ is in control?
To circumvent these challenges, we train \emph{exclusively} in simulation.
This significantly simplifies the training procedure, but presents a new hurdle: how do we minimize the difference between the sensory input received by the controller in simulation and reality?

Our approach to bridging the gap between simulation and reality is to leverage \emph{abstraction}~\cite{muller2018driving}. Rather than operating on raw sensory input, our sensorimotor controller operates on an intermediate representation produced by a perception module~\cite{Zhou2019DoesCV}. This intermediate representation is more consistent across simulation and reality than raw visual input.

We now formally show that training a network on abstractions of sensory input reduces the gap between simulation and reality.
Let $M(\bm{z}\mid \bm{s}),L(\bm{z}\mid \bm{s}) \colon \mathbb{S} \to \mathbb{O}$ denote the observation models in the real world and in simulation, respectively.
Such models describe how an on-board sensor measurement $\bm{z}$ senses a state $\bm{s}$.
We further define $\pi_r = \mathbb{E}_{\bm{o}_r \sim M(\bm{s})} [\pi(\bm{o}_r[k])]$ and $\pi_{s} = \mathbb{E}_{\bm{o}_s \sim L(\bm{s})} [\pi(\bm{o}_s[k])]$ as the realizations of the policy $\pi$ in the real world and in simulation.
The following lemma shows that, disregarding actuation differences, the distance between the observation models upper-bounds the gap in performance in simulation and reality.
\begin{lemma}
For a Lipschitz-continuous policy $\pi$ the simulation-to-reality gap $J(\pi_r) - J(\pi_s)$ is upper-bounded by
\begin{equation}
    J(\pi_r) - J(\pi_s) \leq C_{\pi_s} K \mathbb{E}_{\rho(\pi_r)} \big[ DW(M, L) \big],
    \label{eq:sim2real}
\end{equation}
where $K$ denotes the Lipschitz constant.
\label{lemma:gap}
\end{lemma}
\begin{proof}
The lemma follows directly from~\eqref{eq:supervised_learning} and the fact that
\begin{align}
     DW(\pi_r, \pi_s) &= \inf_{\gamma \in \Pi(\bm{o}_r, \bm{o}_s)} \mathbb{E}_{(\bm{o}_r, \bm{o}_s)}[d_p(\pi_r, \pi_s)] \nonumber\\
                      & \leq  K \inf_{\gamma \in \Pi(\bm{o}_r, \bm{o}_s)}\mathbb{E}_{(\bm{o}_r, \bm{o}_s)}[d_o(\bm{o}_r, \bm{o}_s)] \nonumber\\
                      &=  K \cdot DW(M, L),\nonumber
\end{align}
where $d_o$ and $d_p$ are distances in observation and action space, respectively.
\end{proof}

We now consider the effect of abstraction of the input observations. Let $f$ be a mapping of the observations such that 
\begin{align}
DW(f(M), f(L)) \leq  DW(M, L).
\label{eq:abstract_observations}
\end{align}
\rebutall{The mapping $f$ is task-dependent and is generally designed~-- with domain knowledge~-- to contain sufficient information to solve the task while being invariant to nuisance factors.
In our case, we use feature tracks as an abstraction of camera frames. The feature tracks are provided by a visual-inertial odometry (VIO) system.
In contrast to camera frames, feature tracks primarily depend on scene geometry, rather than surface appearance.
We also make inertial measurements independent of environmental conditions, such as temperature and pressure, by integration and de-biasing.
As such, our input representations fulfill the requirements of Eq.~\eqref{eq:abstract_observations}.
}

As the following lemma shows, training on such representations reduces the gap between task performance in simulation and the real world. 
\begin{lemma}
A policy that acts on an abstract representation of the observations $\pi_f \colon f(\mathbb{O}) \to \mathbb{U}$ has a lower simulation-to-reality gap than a policy $\pi_o \colon \mathbb{O} \to \mathbb{U}$ that acts on raw observations.
\label{lemma:abstraction}
\end{lemma}
\begin{proof}
The lemma follows directly from~\eqref{eq:sim2real} and~\eqref{eq:abstract_observations}.
\end{proof}

\subsection{Sensorimotor Controller}
\label{sec:agent}

In contrast to the expert policy $\pi^{*}$, the student policy $\pi$ is only provided with onboard sensor measurements from the forward-facing camera and the IMU.
There are three main challenges for the controller to tackle: (i) it should keep track of its state based on the provided inputs, akin to a visual-inertial odometry system~\cite{svo_gtsam, engel2018direct}, (ii) it should be invariant to environments and domains, so as to not require retraining for each scene, and (iii) it should process sensor readings that are provided at different frequencies.

We represent the policy as a neural network that fulfills all of the above requirements. The network consists of three input branches that process visual input, inertial measurements, and the reference trajectory, followed by a multi-layer perceptron that produces actions. The architecture is illustrated in Fig.~\ref{fig:network}.
Similarly to visual-inertial odometry systems~\cite{clark2017vinet, engel2018direct, svo_gtsam}, we provide the network with a representation of the platform state by supplying it with a history of length $L=8$ of visual and inertial information.

To ensure that the learned policies are scene- and domain-independent, we provide the network with appropriate abstractions of the inputs instead of directly feeding raw inputs. We design these abstractions to contain sufficient information to solve the task while being invariant to environmental factors that are hard to simulate accurately and are thus unlikely to transfer from simulation to reality.

The distribution of raw IMU measurements depends on the exact sensor as well as environmental factors such as pressure and temperature.
Instead of using the raw measurements as input to the policy, we preprocess the IMU signal by applying bias subtraction and gravity alignment.
Modern visual-inertial odometry systems perform a similar pre-integration of the inertial data in their optimization backend \cite{qin2018vins}.
The resulting inertial signal contains the estimated platform velocity, orientation, and angular rate.

\begin{table*}[htb!]
  \robustify\bfseries
  \newcolumntype{R}[0]{S[table-format=3(3),table-number-alignment = center,separate-uncertainty = true,detect-weight,detect-mode]}
    \centering
    \small
    \setlength{\tabcolsep}{3pt}
    \begin{tabular}{l|ccc|Rc|Rc|Rc|Rc}
    \toprule
         Maneuver &
         \multicolumn{3}{c|}{Input} & \multicolumn{2}{c}{Power Loop} & \multicolumn{2}{c}{Barrel Roll} & \multicolumn{2}{c}{Matty Flip} & \multicolumn{2}{c}{Combo} \\
         \cmidrule(l){2-4} \cmidrule(l){5-12}
      & Ref & IMU & FT   &  {Error ($\downarrow$)} & Success ($\uparrow$) &   {Error ($\downarrow$)} & Success ($\uparrow$) &  {Error ($\downarrow$)} & Success ($\uparrow$) &  {Error ($\downarrow$)} & Success ($\uparrow$) \\
        \midrule
     VIO-MPC & \checkmark & \checkmark & \checkmark& 43(14) & \textbf{100\%} & 79(43) & \textbf{100\%} & 92(41) & \textbf{100\%} & 164(51) & 70\% \\
     Ours (Only Ref) & \checkmark & &  & 250(50) & 20\% & 485(112) & 85\%  & 340(120) & 15\% & $\infty$ & 0 \% \\
     Ours (No IMU) & \checkmark & & \checkmark & 210(100) & 30\% & 543(95) & 85\%  & 380(100) & 20\% & $\infty$ & 0 \%\\
     Ours (No FT) & \checkmark & \checkmark &  & 28(8) & \textbf{100\%} & 64(24) & 95\% & 67(29) & \textbf{100\%} &  134(113) & 85\%  \\
     Ours & \checkmark & \checkmark & \checkmark & \bfseries 24(5) & \textbf{100\%} & \bfseries 58(9) & \textbf{100\%} & \bfseries 53(15) & \textbf{100\%} &  \bfseries 128(57) & \textbf{95\%}  \\
           \bottomrule
    \end{tabular}
    \caption{Comparison of different variants of our approach with the baseline (VIO-MPC) in terms of the average tracking error in centimeters and the success rate. Results were averaged over 20 runs. Agents without access to IMU data perform poorly. An agent that has access only to IMU measurements has a significantly lower tracking error than the baseline. Adding feature tracks further improves tracking performance and success rate, especially for longer and more complicated maneuvers.}
    \label{tab:sim_results}
\end{table*}

We use the history of filtered inertial measurements, sampled at $\SI{100}{\hertz}$, and process them
using temporal convolutions~\cite{bai2018empirical}.
Specifically, the inertial branch consists of a temporal convolutional layer with 128 filters, followed by three temporal convolutions with 64 filters each. A final fully-connected layer maps the signal to a 128-dimensional representation.

Another input branch processes a history of reference velocities, orientations, and angular rates. It has the same structure as the inertial branch.
New reference states are added to the history at a rate of $\SI{50}{\hertz}$.

For the visual signal, we use \emph{feature tracks}, i.e.\ the motion of salient keypoints in the image plane, as an abstraction of the input.
Feature tracks depend on the scene structure, ego-motion, and image gradients, but not on absolute pixel intensities.
At the same time, the information contained in the feature tracks is sufficient to infer the ego-motion of the platform up to an unknown scale. Information about the scale can be recovered from the inertial measurements.
We leverage the computationally efficient feature extraction and tracking frontend of VINS-Mono~\cite{qin2018vins} to generate feature tracks.
The frontend extracts Harris corners~\cite{harris1988combined} and tracks them using the Lucas-Kanade method~\cite{lucas1981iterative}.
We perform geometric verification and exclude correspondences with a distance of more than $2$ pixels from the epipolar line.
We represent each feature track by a 5-dimensional vector that consists of the keypoint position, its displacement with respect to the previous keyframe (both on the rectified image plane), and the number of keyframes that the feature has been tracked (a measure of keypoint quality).

To facilitate efficient batch training, we randomly sample $40$ keypoints per keyframe.
The features are processed by a reduced version of the PointNet architecture proposed in~\cite{ranftl2018deep} before we generate a fixed-size representation at each timestep.
Specifically, we reduce the number of hidden layers from 6 to 4, with $32, 64, 128, 128$ filters, respectively, in order to reduce latency.
The output of this subnetwork is reduced to a $128$-dimensional vector by global average pooling over the feature tracks.
The history of resulting hidden representations is then processed by a temporal convolutional network that has the same structure as the inertial and reference branches.

Finally, the outputs of the individual branches are concatenated and processed by a synchronous multi-layer perceptron with three hidden layers of size $128, 64, 32$.
The final outputs are the body rates and collective thrust that are used to control the platform.

We account for the different input frequencies by allowing each of the input branches to operate asynchronously.
Each branch operates independently from the others by generating an output only when a new input from the sensor arrives.
The multi-layer perceptron uses the latest outputs from the asynchronous branches and operates at $\SI{100}{\hertz}$. It outputs control commands at approximately the same rate due to its minimal computational overhead.

\subsection{Implementation Details}
\label{sec:implementation}

We use the Gazebo simulator to train our policies.
Gazebo can model the physics of quadrotors with high fidelity using the RotorS extension~\cite{Furrer2016}. 
We simulate the AscTec Hummingbird multirotor, which is equipped with a forward-facing fisheye camera. 
The platform is instantiated in a cubical simulated flying space with a side length of 70 meters. An example image is shown in Fig.~\ref{fig:sim2real_fig} (left).

\begin{figure}[t]
    \centering
    \includegraphics[width=0.49\linewidth]{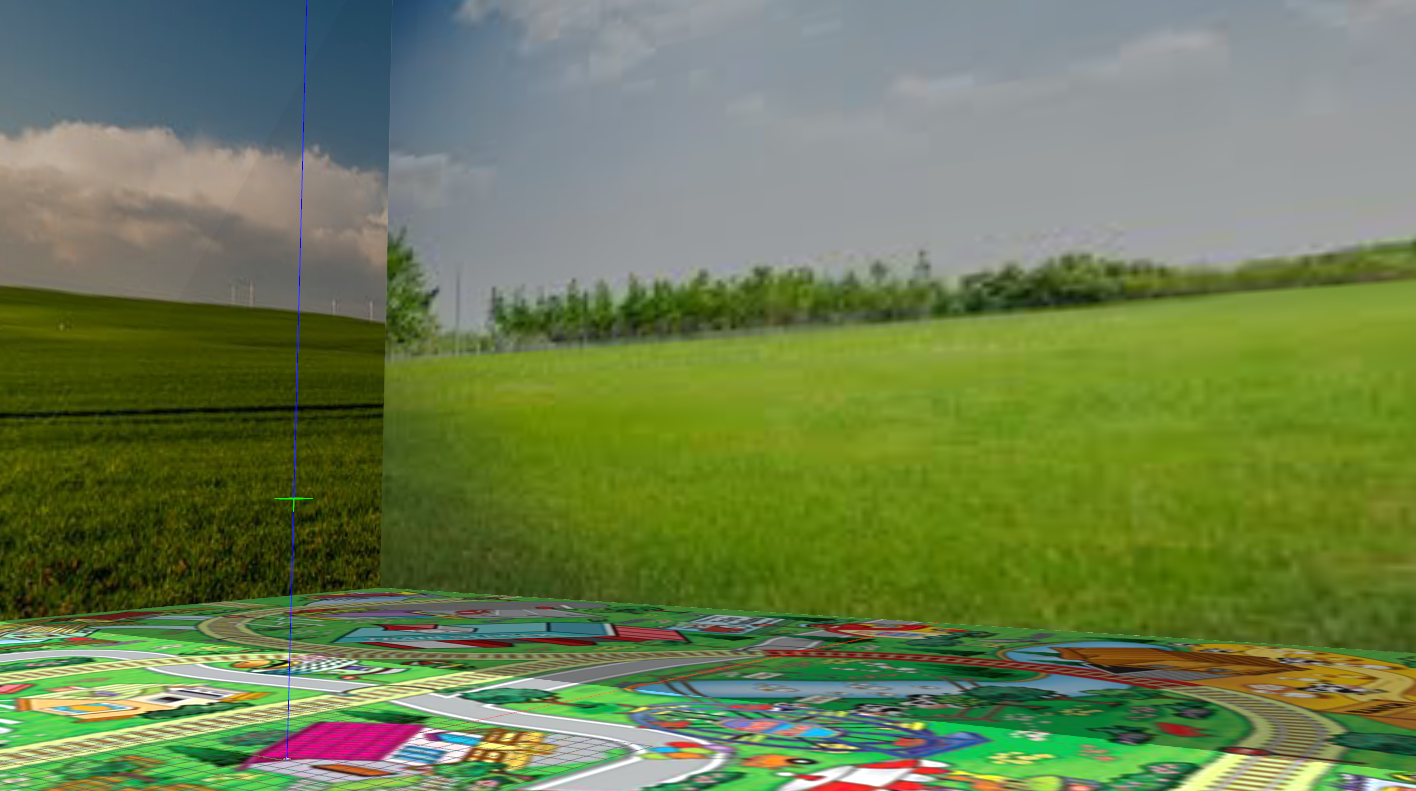}
    \includegraphics[width=0.49\linewidth]{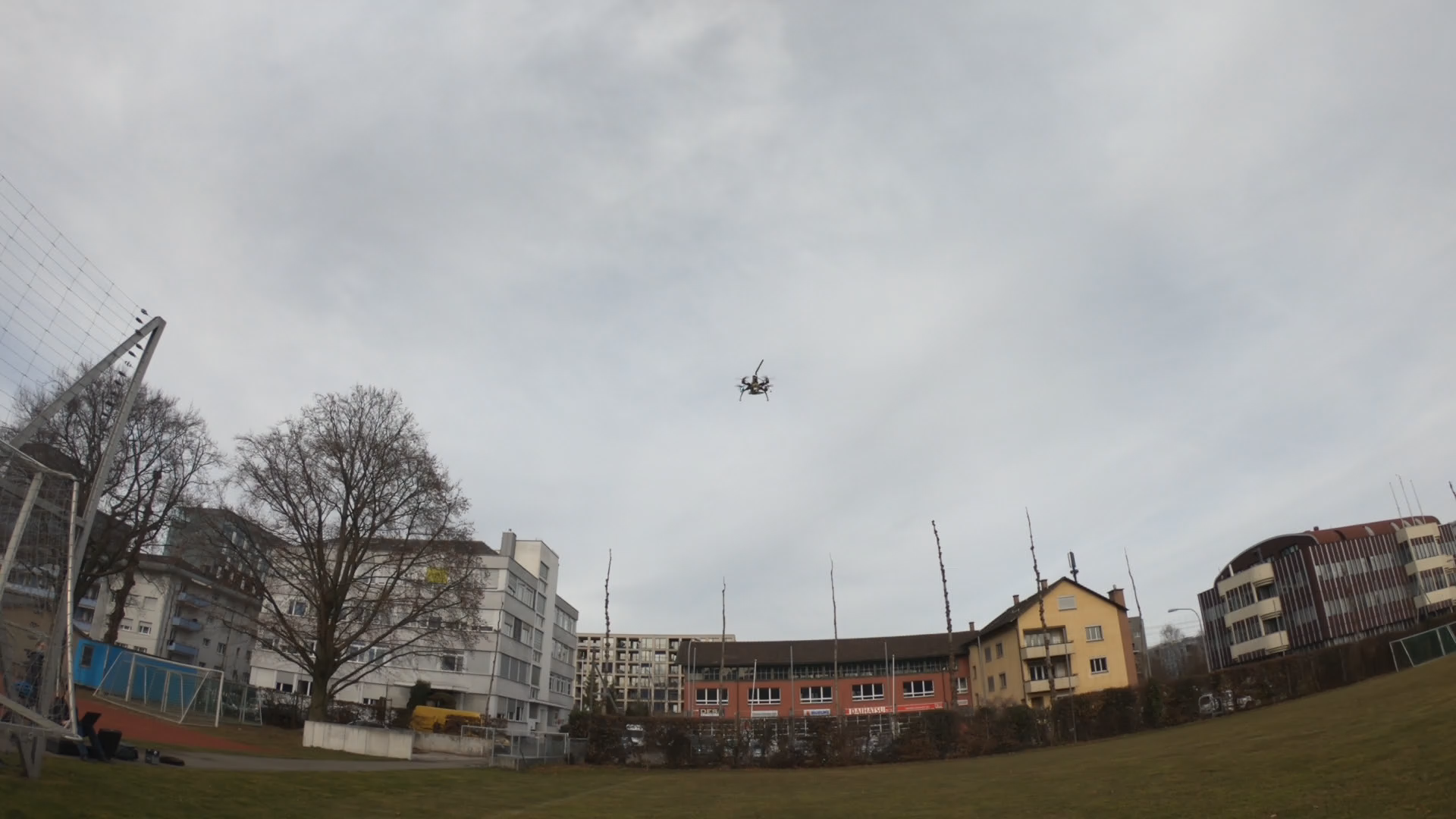}
    \caption{Example images from simulation (left) and the real test environment (right).}
    \label{fig:sim2real_fig}
    \vspace{-4mm}
\end{figure}

For the real-world experiments we use a custom quadrotor that weighs $\SI{1.15}{\kilo\gram}$ and has a thrust-to-weight ratio of 4:1. We use a Jetson TX2 for neural network inference. Images and inertial measurements are provided by an Intel RealSense~T265 camera.

We use an off-policy learning approach. We execute the trained policy, collect rollouts, and add them to a dataset. After 30 new rollouts are added, we train for 40 epochs on the entire dataset. This collect-rollouts-and-train procedure is repeated 5 times: there are 150 rollouts in the dataset by the end. We use the Adam optimizer~\cite{KingmaB14} with a learning rate of $3e-4$. We always use the latest available model for collecting rollouts. 
We execute a student action only if the difference to the expert action is smaller than a threshold $t=1.0$ to avoid crashes in the early stages of training.
We double the threshold $t$ every 30 rollouts.
We perform a random action with a probability $p=30\%$ at every stage of the training to increase the coverage of the state space.
To facilitate transfer from simulation to reality, we randomize the IMU biases and the thrust-to-weight ratio of the platform by up to $10\%$ of their nominal value in every iteration.
We do not perform any randomization of the geometry and appearance of the scene during data collection.
\begin{figure*}[t]
    \centering
    \hspace{-10pt}
    \begin{subfigure}{0.45\linewidth}
        \input{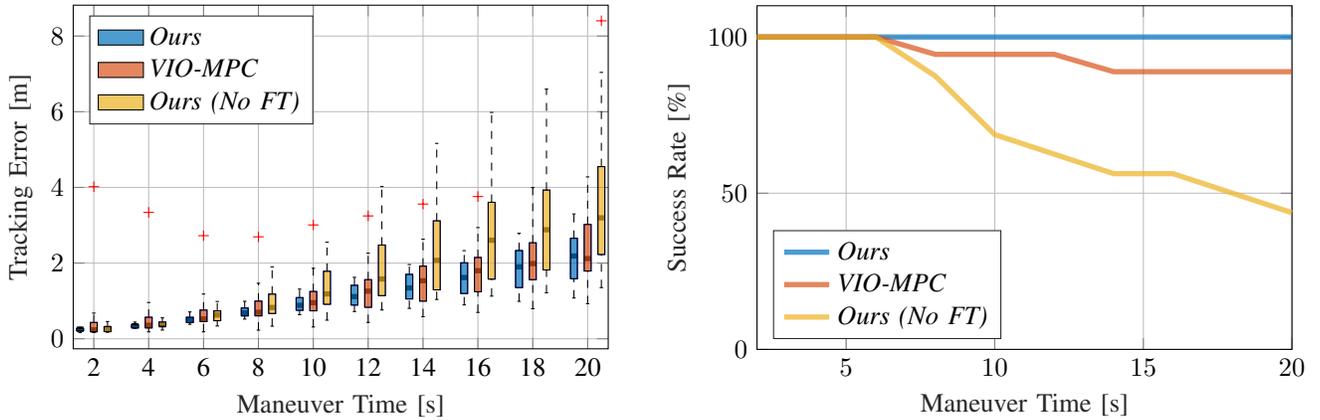}
    \end{subfigure}
    \hspace{10pt}
    \begin{subfigure}{0.45\linewidth}
%
%
\definecolor{mycolor1}{rgb}{0.00000,0.44700,0.74100}%
\definecolor{mycolor2}{rgb}{0.85000,0.32500,0.09800}%
\definecolor{mycolor3}{rgb}{0.92900,0.69400,0.12500}%
\begin{tikzpicture}

\begin{axis}[%
width=2.8in,
height=1.8in,
at={(0.0in,0.0in)},
scale only axis,
xmin=2,
xmax=20,
xlabel style={font=\color{white!15!black}},
xlabel={Maneuver Time [s]},
ymin=0,
ymax=110,
ylabel style={font=\color{white!15!black}},
ylabel={Success Rate [\%]},
axis background/.style={fill=white},
xmajorgrids,
ymajorgrids,
legend style={at={(0.03,0.03)}, anchor=south west, legend cell align=left, align=left, draw=white!15!black}
]
\addplot [color=mycolor1, line width=2.0pt, opacity=0.7]
  table[row sep=crcr]{%
2	100\\
4	100\\
6	100\\
8	100\\
10	100\\
12	100\\
14	100\\
16	100\\
18	100\\
20	100\\
};
\addlegendentry{\textit{Ours}}

\addplot [color=mycolor2, line width=2.0pt, opacity=0.7]
  table[row sep=crcr]{%
2	100\\
4	100\\
6	100\\
8	94.4444444444444\\
10	94.4444444444444\\
12	94.4444444444444\\
14	88.8888888888889\\
16	88.8888888888889\\
18	88.8888888888889\\
20	88.8888888888889\\
};
\addlegendentry{\textit{VIO-MPC}}

\addplot [color=mycolor3, line width=2.0pt, opacity=0.7]
  table[row sep=crcr]{%
2	100\\
4	100\\
6	100\\
8	87.5\\
10	68.75\\
12	62.5\\
14	56.25\\
16	56.25\\
18	50\\
20	43.75\\
};
\addlegendentry{\textit{Ours (No FT)}}

\end{axis}

\end{tikzpicture}%
    \end{subfigure}
    \caption{Tracking error (left) and success rate (right) over time when a maneuver is executed repeatedly in simulation. The controllers were trained to complete the maneuver for six seconds and generalize well to longer sequences. Our learned controller, which leverages both IMU and visual data, provides consistently good performance without a single failure.
    }
    \label{fig:man_len_generalization}
\end{figure*}

\section{Experiments}

We design our evaluation procedure to address the following questions. Is the presented sensorimotor controller advantageous to a standard decomposition of state estimation and control? What is the role of input abstraction in facilitating transfer from simulation to reality?
Finally, we validate our design choices with ablation studies.

\subsection{Experimental Setup}
We learn sensorimotor policies for three acrobatic maneuvers that are popular among professional drone pilots as well as a policy that consists of a sequence of multiple maneuvers. %
\begin{enumerate}
    \item Power Loop: Accelerate over a distance of $\SI{4}{\meter}$ to a speed of $\SI{4.5}{\meter\per\second}$ and enter a loop maneuver with a radius of ${r=\SI{1.5}{\meter}}$. 
    \item Barrel Roll: Accelerate over a distance of $\SI{3}{\meter}$ to a speed of $\SI{4.5}{\meter\per\second}$ and enter a roll maneuver with a radius of ${r=\SI{1.5}{\meter}}$.
    \item Matty Flip: Accelerate over a distance of $\SI{4}{\meter}$ to a speed of $\SI{4.5}{\meter\per\second}$ while yawing $\SI{180}{\degree}$ and enter a backwards loop maneuver with a radius of ${r=\SI{1.5}{\meter}}$. 
    \item Combo: This sequence starts with a triple Barrel Roll, followed by a double Power Loop, and ends with a Matty Flip. The full maneuver is executed without stopping between maneuvers.
\end{enumerate}
The maneuvers are listed by increasing difficulty. The trajectories of these maneuvers are illustrated in Fig.~\ref{fig:maneuvers}. They contain high accelerations and fast angular velocities around the body axes of the platform. All maneuvers start and end in the hover condition.

For comparison, we construct a strong baseline by combining visual-inertial odometry~\cite{qin2018vins} and model predictive control~\cite{falanga2018pampc}. Our baseline receives the same inputs as the learned controllers: inertial measurements, camera images, and a reference trajectory.

We define two metrics to compare different approaches. We measure the average root mean square error in meters of the reference position with respect to the true position of the platform during the execution of the maneuver. Note that we can measure this error only for simulation experiments, as it requires exact state estimation. We thus define a second metric, the average success rate for completing a maneuver. In simulation, we define success as not crashing the platform into any obstacles during the maneuver. For the real-world experiments, we consider a maneuver successful if the safety pilot did not have to intervene during the execution and the maneuver was executed correctly.

\subsection{Experiments in Simulation}
We first evaluate the performance for individual maneuvers in simulation.
The results are summarized in Table~\ref{tab:sim_results}.
The learned sensorimotor controller that has access to both visual and inertial data (\textit{Ours}) is consistently the best across all maneuvers.
This policy exhibits a lower tracking error by up to 45\% in comparison to the strong \textit{VIO-MPC} baseline.
The baseline can complete the simpler maneuvers with perfect success rate, but generally has higher tracking error due to drift in state estimation. 
The gap between the baseline and our controller widens for longer and more difficult sequences.

Table~\ref{tab:sim_results} also highlights the relative importance of the input modalities. Policies that only receive the reference trajectories but no sensory input (\textit{Ours (Only Ref)})~-- effectively operating open-loop~-- perform poorly across all maneuvers. Policies that have access to visual input but not to inertial data \textit{(Ours (No IMU))} perform similarly poorly since they do not have sufficient information to determine the absolute scale of the ego-motion. On the other hand, policies that only rely on inertial data for sensing \textit{(Ours (No FT))} are able to safely fly most maneuvers. Even though such controllers only have access to inertial data, they exhibit significantly lower tracking error than the \textit{VIO-MPC} baseline.
\rebutall{
However, the longer the maneuver, the larger the drift accumulated by purely-inertial \textit{(Ours (No FT))} controllers.
When both inertial and visual data is incorporated \textit{(Ours)}, drift is reduced and accuracy improves.
For the longest sequence \textit{(Combo)}, the abstracted visual input raises the success rate by 10 percentage points.
}

Fig.~\ref{fig:man_len_generalization} analyzes the evolution of tracking errors and success rates of different methods over time.
For this experiment, we trained a policy to repeatedly fly barrel rolls for four seconds. We evaluate robustness and generalization of the learned policy by flying the maneuver for up to 20 seconds at test time. The results again show that (abstracted) visual input reduces drift and increases robustness.
The controller that has access to both visual and inertial data \emph{(Ours)} is able to perform barrel rolls for 20 seconds without a single failure.

\begin{table}[t]
  \robustify\bfseries
  \newcolumntype{R}[0]{S[table-format=2(2),table-number-alignment = center,separate-uncertainty = true,detect-weight,detect-mode]}
    \centering
    \scalebox{0.82}{
    \begin{tabular}{c | Rc | Rc | Rc}
    \toprule
    Input & \multicolumn{2}{c}{Train} & \multicolumn{2}{c}{Test 1} & \multicolumn{2}{c}{Test 2}\\
    \cmidrule(l){2-7}
    & {Error ($\downarrow$)} & Success ($\uparrow$) & {Error ($\downarrow$)} & Success ($\uparrow$) & {Error ($\downarrow$)} & Success ($\uparrow$)\\
    \midrule
    Image & 90(32) & 80\% & $\infty$ & 0\% & $\infty$ & 0\%\\
    Ours & \bfseries 53(15)& \textbf{100\%} & \bfseries 58(18) & \textbf{100\%} & \bfseries 61(11) & \textbf{100\%}\\
    \bottomrule
    \end{tabular}}
    \caption{Sim-to-sim transfer for different visual input modalities. Policies that directly rely on images as input do not transfer to scenes with novel appearance (Test 1, Test 2). Feature tracks enable reliable transfer. Results are averaged over $10$ runs.}
    \label{tab:sim2sim}
    \vspace{-3mm}
\end{table}

To validate the importance of input abstraction, we compare our approach to a network that uses raw camera images instead of feature tracks as visual input.
This network substitutes the PointNet in the input branch with a $5$-layer convolutional network that directly operates on image pixels, but retains the same structure otherwise.
We train this network on the Matty Flip and evaluate its robustness to changes in the background images.
The results are summarized in Table~\ref{tab:sim2sim}.
In the training environment, the image-based network has a success rate of only 80\%, with a 58\% higher tracking error than the controller that receives an abstraction of the visual input in the form of feature tracks \emph{(Ours)}. We attribute this to the higher sample complexity of learning from raw pixels~\cite{Zhou2019DoesCV}.
Even more dramatically, the image-based controller fails completely when tested with previously unseen background images (\emph{Test 1}, \emph{Test 2}). (For backgrounds, we use randomly sampled images from the COCO dataset~\cite{Lin2014}.)
In contrast, our approach maintains a 100\% success rate in these conditions.

\begin{table}[t!]
    \centering
    \small
    \setlength{\tabcolsep}{3pt}
    \scalebox{0.9}{
    \begin{tabular}{lccc}
    \toprule
         Maneuver & Power Loop & Barrel Roll & Matty Flip \\
        \midrule
        Ours (No FT) & 100\% & 90\% & 100\%\\
        Ours & 100\% & 100\% & 100\% \\
        \bottomrule
    \end{tabular}}
    \caption{Success rate across 10 runs on the physical platform.}
    \label{tab:real_world}
\end{table}

\subsection{Deployment in the Physical World}
We now perform direct simulation-to-reality transfer of the learned controllers.
We use exactly the same sensorimotor controllers that were learned in simulation and quantitatively evaluated in Table~\ref{tab:sim_results} to fly a physical quadrotor, with no fine-tuning.
Despite the differences in the appearance of simulation and reality (see Fig.~\ref{fig:sim2real_fig}), the abstraction scheme we use facilitates successful deployment of simulation-trained controllers in the physical world. The controllers are shown in action in the supplementary video.

We further evaluate the learned controllers with a series of quantitative experiments on the physical platform.
The success rates of different maneuvers are shown in Table~\ref{tab:real_world}.
Our controllers can fly all maneuvers with no intervention.
An additional experiment is presented in Fig.~\ref{fig:multi_loops}, where a controller that was trained for a single loop was tested on repeated execution of the maneuver with no breaks. The results indicate that using all input modalities, including the abstracted visual input in the form of feature tracks, enhances robustness.

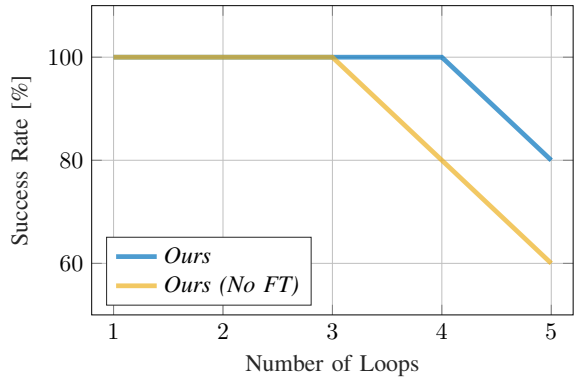
\begin{figure}[t]
    \centering
    \scalebox{0.9}{
%
%
\definecolor{mycolor1}{rgb}{0.00000,0.44700,0.74100}%
\definecolor{mycolor2}{rgb}{0.92900,0.69400,0.12500}%
\begin{tikzpicture}

\begin{axis}[%
width=2.8in,
height=1.8in,
at={(0.0in,0.0in)},
scale only axis,
xmin=0.8,
xmax=5.2,
xlabel style={font=\color{white!15!black}},
xlabel={Number of Loops},
ymin=50,
ymax=110,
ylabel style={font=\color{white!15!black}},
ylabel={Success Rate [\%]},
axis background/.style={fill=white},
xmajorgrids,
ymajorgrids,
legend style={at={(0.03,0.03)}, anchor=south west, legend cell align=left, align=left, draw=white!15!black}
]
\addplot [color=mycolor1, line width=2.0pt, opacity=0.7]
  table[row sep=crcr]{%
1	100\\
2	100\\
3	100\\
4	100\\
5	80\\
};
\addlegendentry{\textit{Ours}}

\addplot [color=mycolor2, line width=2.0pt, opacity=0.7]
  table[row sep=crcr]{%
1	100\\
2	100\\
3	100\\
4	80\\
5	60\\
};
\addlegendentry{\textit{Ours (No FT)}}

\end{axis}
\end{tikzpicture}%
}
    \vspace{-1mm}
    \caption{Number of successful back-to-back Power Loops on the physical quadrotor before the human expert pilot had to intervene. Results are averaged over 5 runs.}
    \label{fig:multi_loops}
\vspace{-3mm}
\end{figure}
\section{Conclusion}
Our approach is the first to enable an autonomous flying machine to perform a wide range of acrobatics maneuvers that are highly challenging even for expert human pilots.
The approach relies solely on onboard sensing and computation, and leverages sensorimotor policies that are trained entirely in simulation.
We have shown that designing appropriate abstractions of the input facilitates direct transfer of the policies from simulation to physical reality.
The presented methodology is not limited to autonomous flight and can enable progress in other areas of robotics.

\section*{Acknowledgement}
This work was supported by the Intel Network on Intelligent Systems, the National Centre of Competence in Research Robotics (NCCR) through the Swiss National Science Foundation, and the SNSF-ERC Starting Grant.
The authors especially thank Thomas L\"angle, Christian Pfeiffer, Manuel Sutter and Titus Cieslewski for their contributions to the design and help with the experiments. 

\newpage
\balance
\bibliographystyle{unsrtnat}
\bibliography{references}

\end{document}